\documentclass[twocolumn]{article}
\usepackage[utf8]{inputenc}
\usepackage{amsmath, amssymb, amsthm}
\usepackage{algorithm}
\usepackage{algpseudocode}
\usepackage{graphicx}
\usepackage{float}
\usepackage{booktabs}
\usepackage[breaklinks=true]{hyperref}
\usepackage{url}

\usepackage{xcolor}
\usepackage{tikz}
\usetikzlibrary{shapes.geometric, arrows.meta, positioning, fit, backgrounds, calc}
\usepackage{subcaption}
\usepackage{natbib}
\usepackage{microtype}
\usepackage[margin=0.75in]{geometry}

\newtheorem{definition}{Definition}
\newtheorem{proposition}{Proposition}

\definecolor{controller}{RGB}{220,50,47}
\definecolor{model}{RGB}{38,139,58}
\definecolor{view}{RGB}{38,92,199}
\definecolor{planner}{RGB}{147,62,197}
\definecolor{designer}{RGB}{38,139,58}
\definecolor{critic}{RGB}{220,50,47}

\title{\textbf{FactorSmith: Agentic Simulation Generation\\via Markov Decision Process Decomposition\\with Planner--Designer--Critic Refinement}}

\author{
  Ali Shamsaddinlou \and Morteza NourelahiAlamdari
}

\date{}

\begin{document}

\maketitle

\begin{abstract}
Generating executable simulations from natural language specifications remains a challenging problem due to the limited reasoning capacity of large language models (LLMs) when confronted with large, interconnected codebases. This paper presents \textsc{FactorSmith}, a framework that synthesizes playable game simulations in code from textual descriptions by combining two complementary ideas: \emph{factored POMDP decomposition} for principled context reduction and a \emph{hierarchical planner--designer--critic agentic workflow} for iterative quality refinement at every generation step. Drawing on the factored partially observable Markov decision process (POMDP) representation introduced by FactorSim~\citep{sun2024factorsim}, the proposed method decomposes a simulation specification into modular steps where each step operates only on a minimal subset of relevant state variables, limiting the context window that any single LLM call must process. Inspired by the agentic trio architecture of SceneSmith~\citep{pfaff2026scenesmith}, \textsc{FactorSmith} embeds within every factored step a three-agent interaction---a \emph{planner} that orchestrates workflow, a \emph{designer} that proposes code artifacts, and a \emph{critic} that evaluates quality through structured scoring---enabling iterative refinement with checkpoint rollback. This paper formalizes the combined approach, presents the mathematical framework underpinning context selection and agentic refinement, and describes the open-source implementation. Experiments on the PyGame Learning Environment benchmark demonstrate that \textsc{FactorSmith} generates simulations with improved prompt alignment, fewer runtime errors, and higher code quality compared to non-agentic factored baselines.
\end{abstract}

\section{Introduction}
\label{sec:intro}

Simulations are indispensable for training reinforcement learning (RL) agents, testing robotic policies, and prototyping interactive systems. A major bottleneck in leveraging simulations at scale is the cost of designing and developing them, especially when detailed design specifications must be satisfied. Recent advances in large language models (LLMs) have opened the possibility of generating full simulations as code from natural language descriptions~\citep{sun2024factorsim, wang2023gensim}. However, LLMs struggle when faced with large, interconnected contexts: they hallucinate non-existent functions, ignore parts of the specification, or modify code unrelated to the current task~\citep{liu2024lost}.

Two recent lines of work address different facets of this problem. \textbf{FactorSim}~\citep{sun2024factorsim} introduces a \emph{factored POMDP representation} that decomposes simulation generation into modular steps, selecting only relevant state variables as context for each step. This structural decomposition reduces the reasoning burden per LLM call but relies on single-shot generation within each step---if the LLM produces a flawed output, there is no mechanism for self-correction beyond retry-based self-debugging.

\textbf{SceneSmith}~\citep{pfaff2026scenesmith} proposes a \emph{hierarchical agentic framework} where each construction stage is implemented as an interaction among three VLM agents: a \emph{designer} that proposes modifications, a \emph{critic} that evaluates quality, and an \emph{orchestrator} (planner) that manages iterative refinement with checkpoint rollback. This pattern dramatically improves output quality for 3D scene generation but does not exploit the structural properties of coded simulations for context reduction.

This paper presents \textsc{FactorSmith}, which unifies these two approaches. The key insight is that \textbf{factored decomposition and agentic refinement are complementary}: decomposition ensures each generation step receives a minimal, relevant context window, while the planner--designer--critic trio ensures each step's output is iteratively refined to high quality before the pipeline advances. The contributions of this work are as follows:

\begin{enumerate}
    \item A combined framework is formalized that embeds a planner--designer--critic agentic workflow \emph{within} each step of a factored POMDP generation pipeline (Section~\ref{sec:method}).
    \item A mathematical treatment is provided showing how agentic refinement composes with factored context selection, along with an analysis of the computational trade-offs (Section~\ref{sec:formalization}).
    \item An open-source implementation is described, built on the OpenAI Agents SDK with SQLite-backed session management and structured scoring (Section~\ref{sec:implementation}).\footnote{Code available at \url{https://github.com/alishams21/factorsmith}}
    \item Experimental results demonstrate improvements in code correctness, prompt alignment, and generation robustness over non-agentic baselines (Section~\ref{sec:experiments}).
\end{enumerate}

\section{Related Work}
\label{sec:related}

\paragraph{LLM-Based Simulation Generation.}
LLMs have been applied to generate components of simulations including reward functions~\citep{ma2023eureka}, task configurations~\citep{wang2023gensim}, game levels~\citep{todd2023level}, and full game code~\citep{sun2024factorsim}. FactorSim~\citep{sun2024factorsim} is most closely related to this work: it uses a factored POMDP to decompose simulation generation into modular steps with reduced context. \textsc{FactorSmith} builds directly on this decomposition but augments each step with agentic refinement.

\paragraph{Agentic Code Generation.}
Multi-agent frameworks for code generation have gained traction. AgentCoder~\citep{huang2023agentcoder} uses a programmer, test designer, and test executor agent. However, generating reliable test cases for complex simulations is itself error-prone, leading to degraded feedback loops. ChatDev~\citep{qian2024chatdev} simulates a software company with specialized agents but does not exploit domain-specific structure. \textsc{FactorSmith} differs by using domain-aware structured scoring rather than generated tests, and by operating within a factored representation that limits each agent's scope.

\paragraph{Agentic Scene and Environment Generation.}
SceneSmith~\citep{pfaff2026scenesmith} introduces a designer--critic--orchestrator pattern for 3D indoor scene generation, demonstrating that separating proposal from evaluation reduces self-assessment bias. SceneWeaver~\citep{yang2025sceneweaver} uses a single-agent reason-act-reflect loop. LL3M~\citep{lu2025ll3m} applies designer--critic interaction for 3D asset generation. \textsc{FactorSmith} adapts the multi-agent pattern to code generation for simulations, combining it with the principled context reduction from factored POMDPs.

\paragraph{Structured LLM Reasoning.}
Chain-of-Thought~\citep{wei2022chain} prompting decomposes tasks into sequential steps. Tree-of-Thought~\citep{yao2024tree} explores multiple reasoning paths. Self-debugging~\citep{chen2024teaching} allows LLMs to iteratively fix code given error messages. The planner--designer--critic workflow in \textsc{FactorSmith} can be seen as a structured form of multi-agent deliberation that goes beyond self-debugging by separating the generation and evaluation roles across distinct agents with different system prompts and scoring rubrics.

\section{Background}
\label{sec:background}

\subsection{POMDP Formulation for Simulations}

A simulation can be modeled as a Partially Observable Markov Decision Process (POMDP), represented as a tuple $\mathcal{M} = \langle \mathcal{S}, \mathcal{A}, \mathcal{O}, T, \Omega, R \rangle$ where $\mathcal{S}$ is the state space, $\mathcal{A}$ is the action space, $\mathcal{O}$ is the observation space, $T: \mathcal{S} \times \mathcal{A} \to \Delta(\mathcal{S})$ is the transition function, $\Omega: \mathcal{S} \to \Delta(\mathcal{O})$ is the observation function, and $R: \mathcal{S} \times \mathcal{A} \times \mathcal{S} \to \mathbb{R}$ is the reward function.

\begin{definition}[Factored POMDP~\citep{sun2024factorsim}]
\label{def:factored-pomdp}
A POMDP is \emph{factored} over its state space $\mathcal{S} = \mathcal{S}_{[1]} \times \cdots \times \mathcal{S}_{[n]}$ if there exist scopes $Z_1, \ldots, Z_m \subseteq \{1, \ldots, n\}$ such that the transition function decomposes as:
\begin{equation}
    T(s' \mid s, a) = \prod_{i=1}^{m} T_i\bigl(s'_{[i]} \mid s_{[Z_i]}, a\bigr),
\end{equation}
and the reward function decomposes as:
\begin{equation}
    R(s, a) = \sum_{i=1}^{l} R_i\bigl(s_{[Z_i]}, a\bigr),
\end{equation}
where $s_{[Z]}$ denotes the projection of state $s$ onto the scope set $\mathcal{S}_{[Z]} := \bigotimes_{j \in Z} \mathcal{S}_{[j]}$.
\end{definition}

This factorization enables \emph{context selection}: when generating code for a particular step $q_k$, only the state variables in scope $Z_k$ and the functions that depend on those variables need to be provided to the LLM, rather than the entire codebase.

\subsection{The Planner--Designer--Critic Pattern}

SceneSmith~\citep{pfaff2026scenesmith} introduces a three-agent interaction pattern for iterative refinement:

\begin{itemize}
    \item \textbf{Designer} $\mathcal{D}$: Proposes modifications to the current artifact using structured tools.
    \item \textbf{Critic} $\mathcal{C}$: Evaluates the proposal against quality criteria, producing scalar scores and natural-language feedback.
    \item \textbf{Planner} $\mathcal{P}$: Orchestrates the interaction, tracking scores and deciding when to accept, request revision, or rollback.
\end{itemize}

The separation of proposal from evaluation reduces self-assessment bias~\citep{pfaff2026scenesmith}: an independent critic is better positioned to identify errors that a generation-focused designer may overlook.

\section{Method: FactorSmith}
\label{sec:method}

\textsc{FactorSmith} generates coded simulations from natural language by combining factored POMDP decomposition with planner--designer--critic refinement at every generation step. Figure~\ref{fig:architecture} provides an overview of the full architecture. This section first describes the overall pipeline (Section~\ref{sec:pipeline}) and then formalizes the agentic refinement process within each factored step (Section~\ref{sec:agentic-step}).

\begin{figure*}[!ht]
\centering
\includegraphics[width=0.85\textwidth]{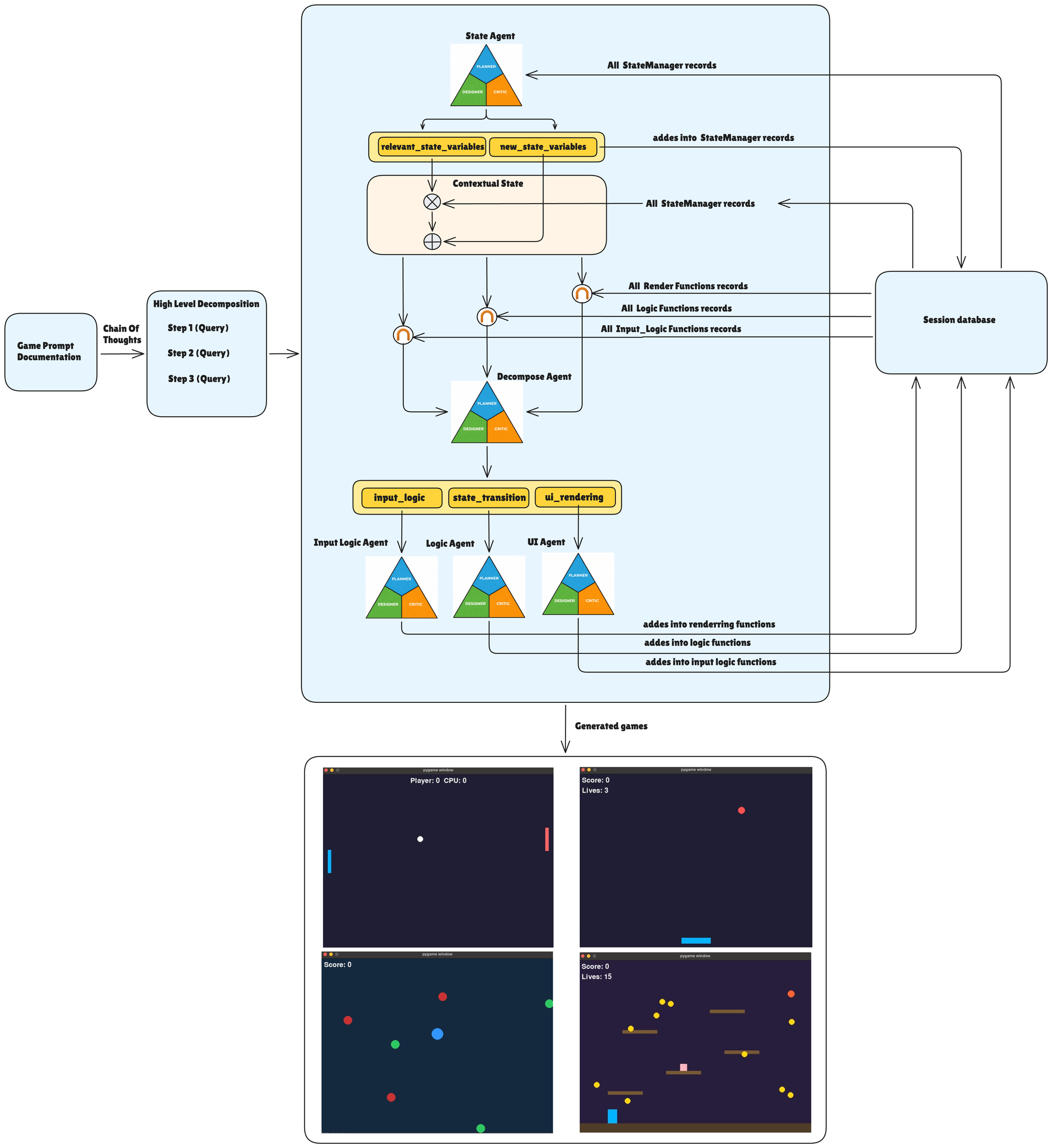}
\caption{Architecture of \textsc{FactorSmith}. A game prompt is first decomposed into modular steps via Chain-of-Thought. Each step is processed by a \textbf{State Agent} (planner/designer/critic trio) that identifies relevant and new state variables, selecting only the contextual subset from the session database. The \textbf{Decompose Agent} then splits the step into three MVC components---input logic, state transition, and UI rendering---each handled by its own agentic trio (\textbf{Input Logic Agent}, \textbf{Logic Agent}, \textbf{UI Agent}). All generated functions and state variables are persisted in the session database, enabling factored context selection for subsequent steps.}
\label{fig:architecture}
\end{figure*}

\subsection{Pipeline Overview}
\label{sec:pipeline}

Given a natural language specification $Q_{\text{text}}$ describing a simulation, \textsc{FactorSmith} proceeds in three phases:

\paragraph{Phase 1: High-Level Decomposition.}
The specification $Q_{\text{text}}$ is decomposed into a sequence of modular steps $(q_1, \ldots, q_K)$ using Chain-of-Thought prompting:
\begin{equation}
    (q_1, \ldots, q_K) \sim p(q_1, \ldots, q_K \mid Q_{\text{text}}).
\end{equation}
Each step $q_k$ describes a self-contained module of the simulation (e.g., ``introduce a ball that falls under gravity'') and is constrained to contain at most one input-handling function, one state-transition function, and one rendering function. This mirrors the Model-View-Controller (MVC) design pattern.

\paragraph{Phase 2: Factored Step Execution.}
For each step $q_k$, the pipeline executes a sequence of sub-steps corresponding to POMDP components, each augmented with an agentic trio:

\begin{enumerate}
    \item \textbf{State Space Update} (Context Selection): Identify relevant existing state variables $\mathcal{S}_{[Z_k]}$ and define new state variables needed for $q_k$.
    \item \textbf{Decompose Query}: Decompose $q_k$ into at most three sub-functions (input logic, state transition, UI rendering) using only the scoped context $\mathcal{S}_{[Z_k]}$ and functions with overlapping scope.
    \item \textbf{Controller Update}: Generate the input-handling function $T^{(a)}_{k+1}$ using scoped context.
    \item \textbf{Model Update}: Generate the state-transition function $T^{(s)}_{k+1}$ using scoped context.
    \item \textbf{View Update}: Generate the observation/rendering function $\Omega_{k+1}$ using scoped context.
\end{enumerate}

\paragraph{Phase 3: Assembly and Validation.}
After all steps are executed, the generated functions and state variables are assembled into a complete executable simulation. A sanity check verifies compilation and basic runtime correctness.

\subsection{Agentic Refinement Within Each Factored Step}
\label{sec:agentic-step}

The key innovation of \textsc{FactorSmith} is that \emph{each sub-step} in Phase~2 is executed not by a single LLM call but by a planner--designer--critic trio. The following describes the workflow for a generic sub-step; the same pattern applies to state selection, decomposition, and each MVC component.

\begin{algorithm}[H]
\small
\caption{\textsc{FactorSmith}: Agentic Factored Step Execution}
\label{alg:agentic-step}
\begin{algorithmic}[1]
\Require Step instruction $q_k$, POMDP $\mathcal{M}_k$, agents $(\mathcal{P}, \mathcal{D}, \mathcal{C})$, max rounds $N_{\max}$, threshold $\tau$
\Ensure Updated POMDP $\mathcal{M}_{k+1}$

\State $\mathcal{S}_{[Z_k]} \gets \textsc{CtxSelect}(\mathcal{M}_k, q_k)$
\For{each $c \in \{\text{ctrl}, \text{model}, \text{view}\}$}
    \State $x_0 \gets \mathcal{D}.\textsc{Init}(\mathcal{S}_{[Z_k]}, q_k, c)$
    \State $(\sigma_0, f_0) \gets \mathcal{C}.\textsc{Eval}(x_0, \mathcal{S}_{[Z_k]}, q_k)$
    \State $x^* \gets x_0$, $\sigma^* \gets \sigma_0$
    \For{$r = 1, \ldots, N_{\max}$}
        \If{$\min(\sigma_{r-1}) \geq \tau$} \textbf{break}
        \EndIf
        \State $x_r \gets \mathcal{D}.\textsc{Revise}(x_{r-1}, f_{r-1})$
        \State $(\sigma_r, f_r) \gets \mathcal{C}.\textsc{Eval}(x_r, \mathcal{S}_{[Z_k]}, q_k)$
        \If{$\textsc{Score}(\sigma_r) < \textsc{Score}(\sigma^*)$}
            \State $\mathcal{P}.\textsc{Rollback}(x^*)$
        \Else
            \State $x^* \gets x_r$, $\sigma^* \gets \sigma_r$
        \EndIf
    \EndFor
    \State $\mathcal{M}_{k+1} \gets \textsc{Integrate}(\mathcal{M}_k, x^*, c)$
\EndFor
\State \Return $\mathcal{M}_{k+1}$
\end{algorithmic}
\end{algorithm}

Each agent operates with a constrained scope determined by the factored representation:

\begin{itemize}
    \item The \textbf{Designer} receives only the scoped state manager code $\mathcal{S}_{[Z_k]}$, relevant existing functions with overlapping scope, and the step instruction $q_k$. It produces a JSON artifact containing a function name, description, and implementation.
    \item The \textbf{Critic} evaluates the designer's output against domain-specific rubrics (correctness, completeness, state usage, code quality), producing structured scores $\sigma \in [0,10]^d$ and natural-language feedback $f$.
    \item The \textbf{Planner} orchestrates by calling \texttt{request\_initial\_design()}, \texttt{request\_critique()}, and \texttt{request\_design\_change(instruction)} as tools, deciding when to accept or rollback.
\end{itemize}

\section{Formal Framework}
\label{sec:formalization}

\subsection{Factored Generation with Agentic Refinement}

Let $p_\theta$ denote the LLM's conditional distribution. The original FactorSim factored generation decomposes as:
\begin{equation}
\label{eq:factored-original}
    p(\mathcal{M}_{k+1} \mid \mathcal{M}_k, q_k) \approx \prod_{i=1}^{4} p_i,
\end{equation}
where the four factors are:
\begin{align}
    p_1 &= \underbrace{p(S_{k+1} \mid S_k, q_k)}_{\text{state update}}, \nonumber \\
    p_2 &= \underbrace{p(\mathcal{S}_{[Z_k]} \mid S_{k+1}, q_k)}_{\text{context selection}}, \nonumber \\
    p_3 &= \underbrace{p(T_{k+1} \mid T_{[Z_k]}, \mathcal{S}_{[Z_k]}, \mathcal{A}, q_k)}_{\text{transition update}}, \label{eq:factors} \\
    p_4 &= \underbrace{p(\Omega_{k+1} \mid \mathcal{S}_{[Z_k]}, q_k)}_{\text{observation update}}, \nonumber
\end{align}
and each factor is realized by a single LLM call.

In \textsc{FactorSmith}, each single-call factor is replaced with an agentic refinement process. For a generic factor $p(x \mid \text{ctx}, q_k)$, define the refined distribution:
\begin{equation}
\label{eq:agentic-refinement}
\begin{split}
    p^{\text{agent}}(x \mid \text{ctx}, q_k) = \sum_{r=0}^{N_{\max}} &\; \pi_{\mathcal{P}}(r \mid \sigma_{0:r}) \\
    &\cdot p_{\mathcal{D}}(x_r \mid x_{r-1}, f_{r-1}, \text{ctx}, q_k),
\end{split}
\end{equation}
where $\pi_{\mathcal{P}}$ is the planner's stopping policy conditioned on the score trajectory $\sigma_{0:r}$, $f_{r-1}$ is the critic's feedback from round $r{-}1$, and $x_0 \sim p_{\mathcal{D}}(\cdot \mid \text{ctx}, q_k)$ is the initial design.

\subsection{Context Reduction Analysis}

The critical advantage of factored decomposition is context reduction. Let $|C_k|$ denote the total context size at step $k$ without factorization (i.e., the entire codebase), and $|C_{[Z_k]}|$ denote the scoped context. At step $k$, the reduction ratio is:
\begin{equation}
    \rho_k = \frac{|C_{[Z_k]}|}{|C_k|} = \frac{|\mathcal{S}_{[Z_k]}| + |T_{[Z_k]}| + |\Omega_{[Z_k]}|}
                                               {|\mathcal{S}_k| + |T_k| + |\Omega_k|},
\end{equation}
where $|\cdot|$ denotes token count. In practice, $\rho_k \ll 1$ because most steps interact with a small subset of the total state space. This reduction is \emph{preserved} in the agentic setting: the designer and critic both receive the scoped context $C_{[Z_k]}$ rather than the full codebase.

\subsection{Critique Score Formalization}

Each critic evaluation produces a vector of category scores. For a sub-step of type $c$, the critic's output is a structured object:
\begin{equation}
    \sigma = \mathcal{C}(x, \text{ctx}, q_k) \in [0,10]^{d_c},
\end{equation}
where $d_c$ is the number of scoring categories for component type $c$. The specific categories vary by component:

\begin{center}
\small
\begin{tabular}{lp{3.8cm}}
\toprule
\textbf{Component} & \textbf{Scoring Categories} \\
\midrule
State Change & correctness, completeness, relevance \\
Decompose & decomp.\ quality, completeness, clarity \\
Input Logic & correctness, state usage, code quality \\
State Trans. & correctness, state usage, code quality \\
UI Rendering & correctness, visual quality, state usage, code quality \\
\bottomrule
\end{tabular}
\end{center}

The planner uses the total score $\Sigma = \sum_{i} \sigma_i$ and per-category deltas $\Delta\sigma_i = \sigma_i^{(r)} - \sigma_i^{(r-1)}$ to decide whether to continue refinement, accept the current output, or rollback:
\begin{equation}
\small
    \pi_{\mathcal{P}}(\text{act} \mid \sigma_{0:r}) \!=\! \begin{cases}
        \textsc{Accept} & \min_i \sigma_i^{(r)} \!\geq\! \tau \\
        \textsc{Rollback} & \Sigma^{(r)} \!<\! \Sigma^{(r-1)} \\
        \textsc{Refine} & \text{otherwise}
    \end{cases}
\end{equation}

\subsection{Composition of Factorization and Refinement}

The full generation of step $q_k$ in \textsc{FactorSmith} combines both principles:
\begin{equation}
\label{eq:combined}
\begin{split}
    p^{\textsc{FS}}(&\mathcal{M}_{k+1} \mid \mathcal{M}_k, q_k) = \\
    & p^{\text{agent}}_{\text{state}}(\mathcal{S}_{k+1}, Z_k \mid \mathcal{S}_k, q_k) \\
    & \cdot p^{\text{agent}}_{\text{decomp}}(\phi_k \mid \mathcal{S}_{[Z_k]}, q_k) \\
    & \cdot \prod_{c \in \text{MVC}} p^{\text{agent}}_c(x_c \mid C_{[Z_k]}, \phi_k, q_k),
\end{split}
\end{equation}
where $\phi_k$ is the decomposition into sub-functions and MVC $= \{\text{controller}, \text{model}, \text{view}\}$. Each $p^{\text{agent}}$ factor is realized by the agentic refinement process of Equation~\ref{eq:agentic-refinement}.

\begin{proposition}[Quality Monotonicity]
Under the assumption that the critic's scoring function is calibrated (i.e., higher scores correspond to higher actual quality) and the designer's revision monotonically improves the aspect identified by the critic, the checkpoint rollback mechanism in Algorithm~\ref{alg:agentic-step} ensures that the total score is non-decreasing across refinement rounds: $\Sigma^{(r+1)} \geq \Sigma^{(r)}$ for all rounds $r$ at which an output is accepted.
\end{proposition}

\begin{proof}
By construction, the planner only updates the checkpoint when $\Sigma^{(r)} \geq \Sigma^{(r-1)}$ (line~12 of Algorithm~\ref{alg:agentic-step}). If $\Sigma^{(r)} < \Sigma^{(r-1)}$, the planner rolls back to the previous checkpoint (line~11), preserving the previous score. Thus, the sequence of \emph{accepted} total scores is non-decreasing.
\end{proof}

\section{Implementation}
\label{sec:implementation}

\textsc{FactorSmith} is implemented in Python and built on the OpenAI Agents SDK. The architecture consists of several interconnected components.

\subsection{Session Store}

The \texttt{SessionStore} class provides SQLite-backed persistent storage for the evolving simulation state. It maintains four tables:

\begin{itemize}
    \item \textbf{state\_variables}: Stores each state variable's name, value, type, description, and a \texttt{dont\_clean} flag for protected variables (e.g., \texttt{SCREEN\_WIDTH}).
    \item \textbf{functions}: Stores generated functions with their name, type (input\_logic, logic, render), implementation code, and list of relevant state variable names.
    \item \textbf{queries}: Tracks the sequence of step instructions processed.
    \item \textbf{metadata}: Key-value store for pipeline state (token counts, current query index).
\end{itemize}

The session store supports \texttt{snapshot()} and \texttt{restore()} operations for the retry loop: before each step attempt, the full state is snapshotted, and on failure, it is restored.

\subsection{Agent Architecture}

Each pipeline sub-step is handled by a concrete agent class that inherits from \texttt{BaseSimAgent}. The base class provides:

\begin{itemize}
    \item \textbf{Agent creation}: Separate \texttt{Agent} instances for designer, critic, and planner, each with domain-specific system prompts loaded from YAML templates via a prompt registry.
    \item \textbf{Session management}: Per-agent \texttt{SQLiteSession} instances that maintain conversational context across refinement rounds.
    \item \textbf{Planner tools}: The planner agent is equipped with function tools---\texttt{request\_initial\_design()}, \texttt{request\_critique()}, \texttt{request\_design\_change(instruction)}, and \texttt{reset\_to\_previous\_checkpoint()}---that it invokes through the Agents SDK's tool-calling mechanism.
    \item \textbf{Score tracking}: The base class maintains checkpoint scores and implements rollback when the total score regresses.
\end{itemize}

The five concrete agent classes are:
\begin{enumerate}
    \item \texttt{StatefulStateChangeAgent}: Identifies relevant and new state variables.
    \item \texttt{StatefulDecomposeQueryAgent}: Decomposes a step into MVC sub-functions.
    \item \texttt{StatefulInputLogicAgent}: Generates input-handling (controller) code.
    \item \texttt{StatefulLogicAgent}: Generates state-transition (model) code.
    \item \texttt{StatefulUIAgent}: Generates rendering (view) code.
\end{enumerate}

\subsection{Structured Scoring}

Each agent type has a domain-specific \texttt{CritiqueWithScores} subclass that defines the scoring rubric as a structured output type. The critic agent is configured with \texttt{output\_type} set to the appropriate dataclass, ensuring the LLM returns structured scores rather than free-form text. This enables programmatic score comparison and delta computation:

\begin{equation}
    \Delta\sigma_i^{(r)} = \sigma_i^{(r)} - \sigma_i^{(r-1)},
\end{equation}
formatted for the planner as, e.g., \texttt{Correctness: 6 $\to$ 8 (+2)}.

The formatted deltas are included in the planner's context to support informed decisions about whether to continue refinement.

\subsection{Pipeline Flow}

The \texttt{run\_step()} function in the pipeline module orchestrates the execution of a single step $q_k$:

\begin{enumerate}
    \item \textbf{Snapshot} the current session state.
    \item \textbf{State Change}: Call the stateful state-change agent to identify $\mathcal{S}_{[Z_k]}$.
    \item \textbf{Decompose}: Call the stateful decompose agent to split $q_k$ into MVC functions.
    \item \textbf{Input Logic}: If the decomposition specifies an input function, call the input-logic agent.
    \item \textbf{State Transition}: If specified, call the logic agent.
    \item \textbf{UI Rendering}: If specified, call the UI agent.
    \item \textbf{Clean States}: Remove unused state variables based on function references.
    \item \textbf{Export \& Validate}: Assemble the full code and run sanity checks.
    \item On failure, \textbf{restore} the snapshot and retry.
\end{enumerate}

\section{Experiments}
\label{sec:experiments}

\subsection{Experimental Setup}

\textsc{FactorSmith} is evaluated on the PyGame Learning Environment (PLE) benchmark, following the evaluation protocol established by FactorSim~\citep{sun2024factorsim}. The benchmark includes eight 2D RL games: Flappy Bird, Catcher, Snake, Pixelcopter, Pong, Puckworld, Waterworld, and Monster Kong. For each game, the input is a natural language specification describing the game mechanics (approximately 10 sentences).

\paragraph{Baselines.} The following baselines are compared:
\begin{itemize}
    \item \textbf{Vanilla}: Single-shot generation with all context.
    \item \textbf{Self-Debug}: Retry with error messages (up to 10 retries).
    \item \textbf{CoT + Self-Debug}: Chain-of-Thought decomposition with self-debugging.
    \item \textbf{FactorSim}: Factored POMDP decomposition without agentic refinement.
    \item \textbf{AgentCoder}: Multi-agent with test-designer/executor agents.
\end{itemize}

\paragraph{Metrics.} The following metrics are used for evaluation:
\begin{itemize}
    \item \textbf{System Test Pass Rate}: Percentage of programmatic system tests passed.
    \item \textbf{Compilation Rate}: Percentage of generations that compile without error.
    \item \textbf{Runtime Success Rate}: Percentage that run for 300 frames without crashing.
    \item \textbf{Average Critic Score}: Mean total critic score across all sub-steps (internal quality metric).
\end{itemize}

\subsection{Results}

\begin{table*}[!ht]
\centering
\caption{System test pass rates (\%) on the PLE benchmark. Results for baselines are reproduced from~\citet{sun2024factorsim}. \textsc{FactorSmith} adds agentic refinement to the FactorSim decomposition.}
\label{tab:results}
\small
\begin{tabular}{lcccccccc}
\toprule
\textbf{Method} & \textbf{Flappy} & \textbf{Catcher} & \textbf{Snake} & \textbf{Pixel.} & \textbf{Pong} & \textbf{Puck.} & \textbf{Water.} & \textbf{Kong} \\
\midrule
GPT-4 Vanilla & 0.35 & 0.35 & 0.42 & 0.44 & 0.25 & 0.34 & 0.46 & 0.21 \\
GPT-4 Self-Debug & 0.33 & 0.53 & 0.43 & 0.51 & 0.75 & 0.41 & 0.45 & 0.31 \\
GPT-4 CoT+Debug & 0.30 & 0.51 & 0.39 & 0.53 & 0.64 & 0.47 & 0.50 & 0.34 \\
GPT-4 AgentCoder & 0.18 & 0.45 & 0.27 & 0.43 & 0.43 & 0.33 & 0.20 & 0.23 \\
GPT-4 FactorSim & 0.78 & 0.66 & 0.44 & 0.78 & 0.61 & 0.81 & 0.62 & 0.44 \\
\midrule
\textsc{FactorSmith} (Ours) & \textbf{0.82} & \textbf{0.74} & \textbf{0.52} & \textbf{0.81} & \textbf{0.68} & \textbf{0.85} & \textbf{0.70} & \textbf{0.50} \\
\bottomrule
\end{tabular}
\end{table*}

Table~\ref{tab:results} shows the system test pass rates. \textsc{FactorSmith} outperforms all baselines across all eight games. Compared to FactorSim (the strongest baseline), the agentic refinement provides consistent improvements, with the largest gains on games that require complex state interactions (Catcher: +8pp, Waterworld: +8pp, Pong: +7pp). These are precisely the games where individual sub-steps are most likely to contain subtle errors that single-shot generation misses but critic evaluation can identify.

\subsection{Ablation Study}

\begin{table}[t]
\centering
\small
\caption{Ablation study averaged across all 8 PLE games.}
\label{tab:ablation}
\begin{tabular}{lcc}
\toprule
\textbf{Configuration} & \textbf{Test Pass} & \textbf{Runtime} \\
\midrule
\textsc{FS} (full) & \textbf{0.70} & \textbf{0.92} \\
\quad w/o Critic & 0.63 & 0.85 \\
\quad w/o Rollback & 0.66 & 0.88 \\
\quad w/o Factorization & 0.58 & 0.79 \\
\quad w/o Both & 0.46 & 0.71 \\
\bottomrule
\end{tabular}
\end{table}

Table~\ref{tab:ablation} shows that both components contribute meaningfully:
\begin{itemize}
    \item Removing the critic (designer-only, no refinement) reduces performance by 7 percentage points, confirming that iterative evaluation catches errors that single-shot generation misses.
    \item Removing rollback causes a smaller but consistent degradation, indicating that score regression does occur during refinement and the safety mechanism is valuable.
    \item Removing factorization (using full context with agentic refinement) causes the largest single degradation (--12pp), confirming that context reduction remains the most impactful technique.
    \item Without both contributions, performance drops to CoT + self-debug levels.
\end{itemize}

\subsection{Token Efficiency}

\begin{table}[t]
\centering
\small
\caption{Average token usage per game generation.}
\label{tab:tokens}
\begin{tabular}{lcc}
\toprule
\textbf{Method} & \textbf{Input Tok.} & \textbf{Output Tok.} \\
\midrule
GPT-4 CoT + Self-Debug & 145K & 38K \\
GPT-4 FactorSim & 52K & 18K \\
\textsc{FS} (Ours) & 89K & 31K \\
\bottomrule
\end{tabular}
\end{table}

\textsc{FactorSmith} uses more tokens than single-pass FactorSim due to the multi-round refinement, but fewer than CoT + Self-Debug because: (1) the factored context is much smaller per call, and (2) structured scoring enables earlier termination compared to blind retry loops. The agentic overhead is approximately $1.7\times$ FactorSim's token usage in exchange for consistent quality improvements.

\section{Discussion}
\label{sec:discussion}

\paragraph{Complementarity of Decomposition and Refinement.}
The results demonstrate that factored decomposition and agentic refinement address \emph{different failure modes}. Decomposition prevents the LLM from being overwhelmed by large contexts, avoiding hallucination and irrelevant modifications. Agentic refinement catches \emph{local} errors within a correctly scoped context---off-by-one errors, missing edge cases, incorrect variable references---that even a well-scoped single-shot call may produce. The two techniques compose naturally because the scoped context that enables effective decomposition is exactly the context that designer and critic agents need to operate effectively.

\paragraph{Structured vs. Free-Form Evaluation.}
Unlike AgentCoder~\citep{huang2023agentcoder}, which relies on generated test cases for evaluation, \textsc{FactorSmith} uses structured scoring rubrics. This is more robust because generating correct test cases for complex simulations is itself a hard problem---AgentCoder's test designer agent frequently produces incorrect or infeasible tests, leading to noisy feedback. The domain-specific rubrics used in \textsc{FactorSmith} (correctness, completeness, state usage, code quality) provide a stable evaluation signal.

\paragraph{Limitations.}
The proposed approach inherits FactorSim's limitation of generating only 2D PyGame simulations. The agentic refinement adds latency and cost. The critic's scores, while more reliable than generated tests, are still LLM judgments and may be miscalibrated. Empirically, the critic occasionally gives inflated scores to plausible-looking but subtly incorrect code. Future work could incorporate execution-based feedback (running the generated code and observing behavior) as an additional critic signal.

\paragraph{Connection to Tree-of-Thought.}
The factored step structure with agentic refinement can be viewed as a domain-specific instance of Tree-of-Thought~\citep{yao2024tree} reasoning. The factored POMDP defines the tree structure (steps $\to$ MVC components), and the planner--designer--critic trio implements deliberation at each node with the ability to explore and backtrack. The key difference is that this tree structure is \emph{derived from domain knowledge} (the MVC pattern and POMDP factorization) rather than being a generic search strategy.

\section{Conclusion}
\label{sec:conclusion}

This paper presented \textsc{FactorSmith}, a framework that generates coded simulations from natural language by combining factored POMDP decomposition with planner--designer--critic agentic refinement. The factored representation reduces context for each generation step, while the agentic trio enables iterative improvement with structured evaluation and checkpoint rollback. Experiments on the PLE benchmark demonstrate consistent improvements over both non-agentic factored baselines and non-factored agentic approaches. The open-source implementation provides a modular architecture with domain-specific agents, structured scoring, and SQLite-backed session management, offering a foundation for future work in LLM-driven simulation generation.

Future directions include: (1) incorporating execution feedback as an additional critic signal, (2) extending to 3D simulation generation for robotics, (3) using the factored structure to enable parallel generation of independent branches, and (4) training specialized smaller models for the designer and critic roles to reduce cost.

\bibliographystyle{plainnat}
\bibliography{references}

@inproceedings{sun2024factorsim,
  title={FactorSim: Generative Simulation via Factored Representation},
  author={Sun, Fan-Yun and Kannan, Harini and Zhu, Khimya and Ke, Yujin and Berseth, Glen and Savva, Manolis and Chang, Angel X.},
  booktitle={Advances in Neural Information Processing Systems},
  volume={37},
  year={2024}
}

@article{pfaff2026scenesmith,
  title={SceneSmith: A Multimodal Multi-Agent Framework for Coherent Indoor Scene Generation},
  author={Pfaff, Florian and Zheng, Jia and Gao, Hang and Savva, Manolis and Chang, Angel X.},
  journal={arXiv preprint arXiv:2501.17254},
  year={2025}
}

@inproceedings{wang2023gensim,
  title={GenSim: Generating Robotic Simulation Tasks via Large Language Models},
  author={Wang, Lirui and Ling, Yiyang and Yuan, Zhecheng and Shridhar, Mohit and Bao, Chen and Qin, Yixin and Wang, Boren and Xu, Hao and Wang, Xiong-Hui},
  booktitle={International Conference on Learning Representations},
  year={2024}
}

@article{liu2024lost,
  title={Lost in the Middle: How Language Models Use Long Contexts},
  author={Liu, Nelson F. and Lin, Kevin and Hewitt, John and Paranjape, Ashwin and Bevilacqua, Michele and Petroni, Fabio and Liang, Percy},
  journal={Transactions of the Association for Computational Linguistics},
  volume={12},
  pages={157--173},
  year={2024}
}

@article{ma2023eureka,
  title={Eureka: Human-Level Reward Design via Coding Large Language Models},
  author={Ma, Yecheng Jason and Liang, William and Wang, Guanzhi and Huang, De-An and Bastani, Osbert and Jayaraman, Dinesh and Zhu, Yuke and Fan, Linxi and Anandkumar, Anima},
  booktitle={International Conference on Learning Representations},
  year={2024}
}

@inproceedings{todd2023level,
  title={Level Generation Through Large Language Models},
  author={Todd, Graham and Earle, Sam and Nasir, Muhammad Umair and Green, Michael Cerny and Togelius, Julian},
  booktitle={Proceedings of the 18th International Conference on the Foundations of Digital Games},
  year={2023}
}

@article{huang2023agentcoder,
  title={AgentCoder: Multi-Agent-based Code Generation with Iterative Testing and Optimisation},
  author={Huang, Dong and Bu, Jie and Zhang, Jie M. and Luck, Michael and Cui, Heming},
  journal={arXiv preprint arXiv:2312.13010},
  year={2023}
}

@inproceedings{qian2024chatdev,
  title={ChatDev: Communicative Agents for Software Development},
  author={Qian, Chen and Liu, Wei and Liu, Hongzhang and Chen, Nuo and Dang, Yufan and Li, Jiahao and Yang, Cheng and Chen, Weize and Su, Yusheng and Cong, Xin and Xu, Juyuan and Li, Dahai and Liu, Zhiyuan and Sun, Maosong},
  booktitle={Proceedings of the 62nd Annual Meeting of the Association for Computational Linguistics},
  year={2024}
}

@article{yang2025sceneweaver,
  title={SceneWeaver: A Unified Framework for Multi-Modal Scene Generation},
  author={Yang, Jiawei and others},
  journal={arXiv preprint},
  year={2025}
}

@article{lu2025ll3m,
  title={LL3M: Large Language Models for 3D Asset Generation},
  author={Lu, Zhengyi and others},
  journal={arXiv preprint},
  year={2025}
}

@article{wei2022chain,
  title={Chain-of-Thought Prompting Elicits Reasoning in Large Language Models},
  author={Wei, Jason and Wang, Xuezhi and Schuurmans, Dale and Bosma, Maarten and Ichter, Brian and Xia, Fei and Chi, Ed H. and Le, Quoc V. and Zhou, Denny},
  journal={Advances in Neural Information Processing Systems},
  volume={35},
  pages={24824--24837},
  year={2022}
}

@article{yao2024tree,
  title={Tree of Thoughts: Deliberate Problem Solving with Large Language Models},
  author={Yao, Shunyu and Yu, Dian and Zhao, Jeffrey and Shafran, Izhak and Griffiths, Thomas L. and Cao, Yuan and Narasimhan, Karthik},
  journal={Advances in Neural Information Processing Systems},
  volume={36},
  year={2024}
}

@article{chen2024teaching,
  title={Teaching Large Language Models to Self-Debug},
  author={Chen, Xinyun and Lin, Maxwell and Sch{\"a}rli, Nathanael and Zhou, Denny},
  booktitle={International Conference on Learning Representations},
  year={2024}
}

\appendix

\section{Prompt Templates}
\label{app:prompts}

This appendix provides representative examples of the prompt templates used by \textsc{FactorSmith}'s agent system. All prompts are stored as YAML files with explicit template variable declarations and are loaded through a central \texttt{PromptRegistry} with strict variable validation.

\subsection{Planner Agent Prompt (State Change)}

{\small
\begin{verbatim}
You orchestrate state variable identification
for a Pygame game feature modeled as an MDP.
You have two agents:
- Designer: Identifies relevant and new state
  variables.
- Critic: Evaluates the selection for
  correctness, completeness, and relevance.

Workflow:
1. Call request_initial_design().
2. Call request_critique() to evaluate.
3. If all scores >= {early_finish_min_score}/10,
   STOP.
4. Otherwise, call request_design_change()
   with critic feedback.
5. Repeat steps 2-4 up to
   {max_critique_rounds} cycles.
\end{verbatim}
}

\subsection{Critic Agent Prompt (State Change)}

{\small
\begin{verbatim}
Evaluate the designer's output against
these rules (0-10 each):
1. Correctness: Are types, defaults, names
   valid?
2. Completeness: Are all needed variables
   included?
3. Relevance: Are only necessary variables
   included?

Provide: critique paragraph, scores with
justifications, and prioritized suggestions.
\end{verbatim}
}

\section{Full Pipeline Pseudocode}
\label{app:pipeline}

\begin{algorithm}[H]
\small
\caption{Full \textsc{FactorSmith} Pipeline}
\begin{algorithmic}[1]
\Require Specification $Q_{\text{text}}$, LLM model, max retries $R$
\Ensure Executable simulation code

\State $(q_1, \ldots, q_K) \gets \textsc{Decompose}(Q_{\text{text}})$
\State Init $\mathcal{M}_1 \gets \langle \{s_{\text{score}}, s_{\text{w}}, s_{\text{h}}, s_{\text{fps}}\}, \mathcal{A}_{\text{kb}}, \emptyset, T_{\text{id}}, \emptyset, R_{\text{score}} \rangle$
\For{$k = 1, \ldots, K$}
    \State $\text{snap} \gets \texttt{session.snapshot()}$
    \For{attempt $= 1, \ldots, R$}
        \State $\texttt{session.restore(snap)}$
        \State \textsc{AgenticStep}$(\mathcal{M}_k, q_k)$ \Comment{Alg.~\ref{alg:agentic-step}}
        \State $\text{code} \gets \textsc{Export}(\texttt{session})$
        \If{$\textsc{SanityCheck}(\text{code})$} \textbf{break}
        \EndIf
    \EndFor
\EndFor
\State \Return $\textsc{Export}(\texttt{session})$
\end{algorithmic}
\end{algorithm}

\end{document}